\documentclass[twoside]{article}

\usepackage[accepted]{aistats2023}
\usepackage[round]{natbib}

\usepackage{tikz}
\usetikzlibrary{positioning}
\usepackage[utf8]{inputenc}
\usepackage[english]{babel}
\usepackage{amsmath}
\usepackage{amsfonts}
\usepackage{amssymb}
\usepackage{comment}
\usepackage{enumitem}
\usepackage{authblk}
\usepackage{mathtools}
\usepackage{algorithm}
\usepackage{algorithmic}

\usepackage{color}
\usepackage{makecell}
\usepackage{graphicx}
\usepackage{float}
\usepackage{bbm}
\usepackage{tabularx}
\usepackage[toc,page]{appendix}
\usepackage{array,multirow}

\usepackage{amsthm}
\usepackage[breaklinks,hidelinks]{hyperref}
\usepackage{caption}
\usepackage{subcaption}

\theoremstyle{definition}
\newtheorem{definition}{Definition}[section]
\newtheorem{example}[definition]{Example}

\newtheorem{lemma}[definition]{Lemma}
\newtheorem{theo}[definition]{Theorem}
\newtheorem{corollar}[definition]{Corollary}
\newtheorem{remark}[definition]{Remark}

\usepackage[utf8]{inputenc} 
\usepackage[T1]{fontenc}    
\usepackage{hyperref}       
\usepackage{url}            
\usepackage{booktabs}       
\usepackage{amsfonts}       
\usepackage{nicefrac}       
\usepackage{microtype}      
\usepackage{xcolor}         

\newcommand\abs[1]{\left|#1\right|}

\begin{document}

\twocolumn[
\aistatstitle{Unifying local and global model explanations by functional decomposition of low dimensional structures}

\aistatsauthor{ Munir Hiabu  \And Joseph T. Meyer  \And  Marvin N.~Wright }

\aistatsaddress{
  University of Copenhagen\\
   \And    
  Heidelberg University \\
  \And  
  Leibniz Institute for Prevention\\ Research and Epidemiology – BIPS,\\ University of Bremen,\\University of Copenhagen
  } 
  ]

\begin{abstract}
       We consider a global representation of a regression or classification function by decomposing it into the sum of main   and interaction components of arbitrary order. We propose a new identification constraint that allows for the extraction of interventional SHAP values and partial dependence plots, thereby unifying local and global explanations. 
        With our proposed identification, a feature's partial dependence plot corresponds to the main effect term plus the intercept. The interventional SHAP value of feature $k$ is a weighted sum of the main component and all interaction components that include $k$, with the weights given by the reciprocal of the component's dimension. This brings a new perspective to local explanations such as SHAP values which were previously motivated by game theory only.
       We show that the decomposition can be used to reduce direct and indirect bias by removing all components that include a protected feature. Lastly, we motivate a new measure of feature importance. In principle, our proposed functional decomposition can be applied to any machine learning model, but exact calculation is only feasible for low-dimensional structures or ensembles of those. We provide an algorithm and efficient implementation for gradient-boosted trees (xgboost) and random planted forest. Conducted experiments suggest that our method provides meaningful explanations and reveals interactions of higher orders. 
       The proposed methods are implemented in an R package, available at \url{https://github.com/PlantedML/glex}.

\end{abstract}

\section{INTRODUCTION}

In the early years of machine learning interpretability research, the focus was mostly on single-value global feature importance methods that assign a single importance value to each feature. More recently, the attention has shifted towards local interpretability methods, which provide explanations for individual observations or predictions. Popular examples of the latter are LIME \citep{ribeiro2016should} and SHAP \citep{shapley1953value,lundberg2017unified}. The major reason for this shift is that local methods provide a more comprehensive picture than single-value global methods, most importantly in presence of nonlinear effects and interactions. This, however, neglects the fact that global methods can be more than single-value methods: Ideally, a global method provides useful information about the entire regression or classification function by providing an explanation for each feature and each interaction effect of arbitrary order, relative to the values they take on. As with local methods, this gives us an explanation for each observation. The crucial difference is that two observations which have a set of feature values in common receive the same explanation for main and interaction effects involving exclusively those features. We call a representation of a function with this property \textit{global}. The components of a global explanation are not specific to all feature values of an observation but only to the feature values corresponding to the respective component. This does not only give a more comprehensive picture than local methods but the complete picture.

In summary, we distinguish between three properties of an explanation of a function. 
\begin{itemize}
    \item single-value global: Each feature $j\in \{1,\dots,d\}$ receives a single descriptive value $v_j\in \mathbb R$ which does not depend on $x\in\mathbb{R}^d$.
    \item global: Each subset of features $S\subseteq\{1,\dots,d\}$ receives a descriptive function $m_S: \ \mathbb{R}^S\rightarrow\mathbb{R}$ which only depends on values $x_S=\{x_k: k\in S\}$ and not on other values $x_{-S}=\{x_{j}: j\notin S\}$. 
    \item local: Each subset of features $S\subseteq\{1,\dots,d\}$ receives a descriptive function $\phi_S: \ \mathbb{R}^d\rightarrow\mathbb{R}$ which may depend on all values of  $x\in\mathbb{R}^d$.
\end{itemize}

In this paper, we introduce a global explanation procedure by identifying components in a functional decomposition. We show that the proposed explanation is identical to $q$-interaction SHAP \citep{tsai2022faith}, where $q$ corresponds to the maximal order of interaction present in the model to be analyzed. Hence, we provide a new interpretation of SHAP values which is not game-theoretically motivated. 
\cite{tsai2022faith} argue that it is  practically not feasible to calculate $q$-interaction SHAP exactly because of computational complexity.
However, the authors implicitly assume $q=d$, i.e., the highest order of interaction present in the initial estimator is equal to the number of features.
We argue that this is not the case for many state-of-the-art machine learning algorithms that only fit low dimensional structures or ensembles of those.
We exploit this fact and discuss an implementation that exactly calculates $q$-interaction SHAP for tree-based machine learning models. 
In principle, our results can be applied to any model and our algorithm can be applied to any tree-based model. However, since the number of components grows exponentially with increasing  $q$, exact calculation is only feasible if $q$ is sufficiently small. We provide a fast implementation for \textit{xgboost} \citep{chen2016xgboost} and \textit{random planted forest} \citep{hiabu2020random}. 


As a result, one dimensional contributions $m_{k}$ and two-dimensional contributions  $m_{jk}$ are one and two-dimensional real-valued functions that can be plotted. Furthermore, together with higher order contributions they can be used to decompose simple SHAP values into main effects and all involved interaction effects. Additionally, main and interaction components can be summarized into feature importance values. Beyond explaining feature effects, our proposed decomposition can be used to detect bias in models where LIME and SHAP fail \citep{slack2020fooling} and reduce such bias by removing individual components from the decomposition. 

The proposed methods are implemented in an R package, available at \url{https://github.com/PlantedML/glex}. 
Code to reproduce all figures and tables is available at \url{https://github.com/PlantedML/glex_paper}. 
Some further details and proofs to all lemmata and theorems are provided in the Appendix.

\subsection{Motivating Example}\label{sec:example}
We will give a toy example of how the interplay of correlations and interactions can give rise to misleading 
SHAP values. 
Consider the function $m(x_1,x_2)=x_1+x_2 + 2x_1x_2 $. The interventional SHAP value for the first feature is
$
\phi_1(x_1,x_2)=x_1-E[X_1]+x_1x_2 -E[X_1X_2]+x_1E[X_2]-x_2E[X_1],
$
see Appendix~\ref{generalexpansion}.
If the features are standardized, i.e.,  $X_1$ and $X_2$ have mean zero and variance one, the expression reduces to 
\[
\phi_1(x_1,x_2)=x_1+x_1x_2 -\textrm{corr}(X_1,X_2).
\]
Hence, e.g., if $\textrm{corr}(X_1,X_2)=0.3$, an individual with $x_1=1$ and $x_2=-0.7$ would see a SHAP value of 0 for the first feature:
\[
\phi_1(1,-0.7) = 0.
\] 
This is quite misleading, since clearly $x_1$ has an effect on the response $m$ that is irrespective of the particular value of $x_1$.
The underlying problem is that locally at $(x_1,x_2)=(1,-0.7)$, the main effect contribution and interaction contribution cancel each other out.
Indeed, we will see that the SHAP value  $\phi_1$ can be decomposed into a main effect contribution of $x_1$, which is
 $m_1^\ast(x_1)=x_1-2\textrm{corr}(X_1,X_2)=0.4$  and an interaction contribution of $\{x_1,x_2\}$, which is $0.5 \times m^\ast_{12}(x_1,x_2)=x_1x_2+\textrm{corr}(X_1,X_2)=-0.4$.
Figure~\ref{fig:example} shows SHAP values and the functional decomposition of an \textit{xgboost} model of the function $m(x_1,x_2)$. The SHAP values $\phi_1$ and $\phi_2$ contain main effect contributions $m_1^\ast$ and $m_2^\ast$ as well as half of the interaction contribution $m^\ast_{12}$ each. The functional decomposition separates the contributions $m^\ast_1$, $m^\ast_2$ and $m^\ast_{12}$. 

\begin{figure}
    \centering
    \includegraphics[width=\linewidth]{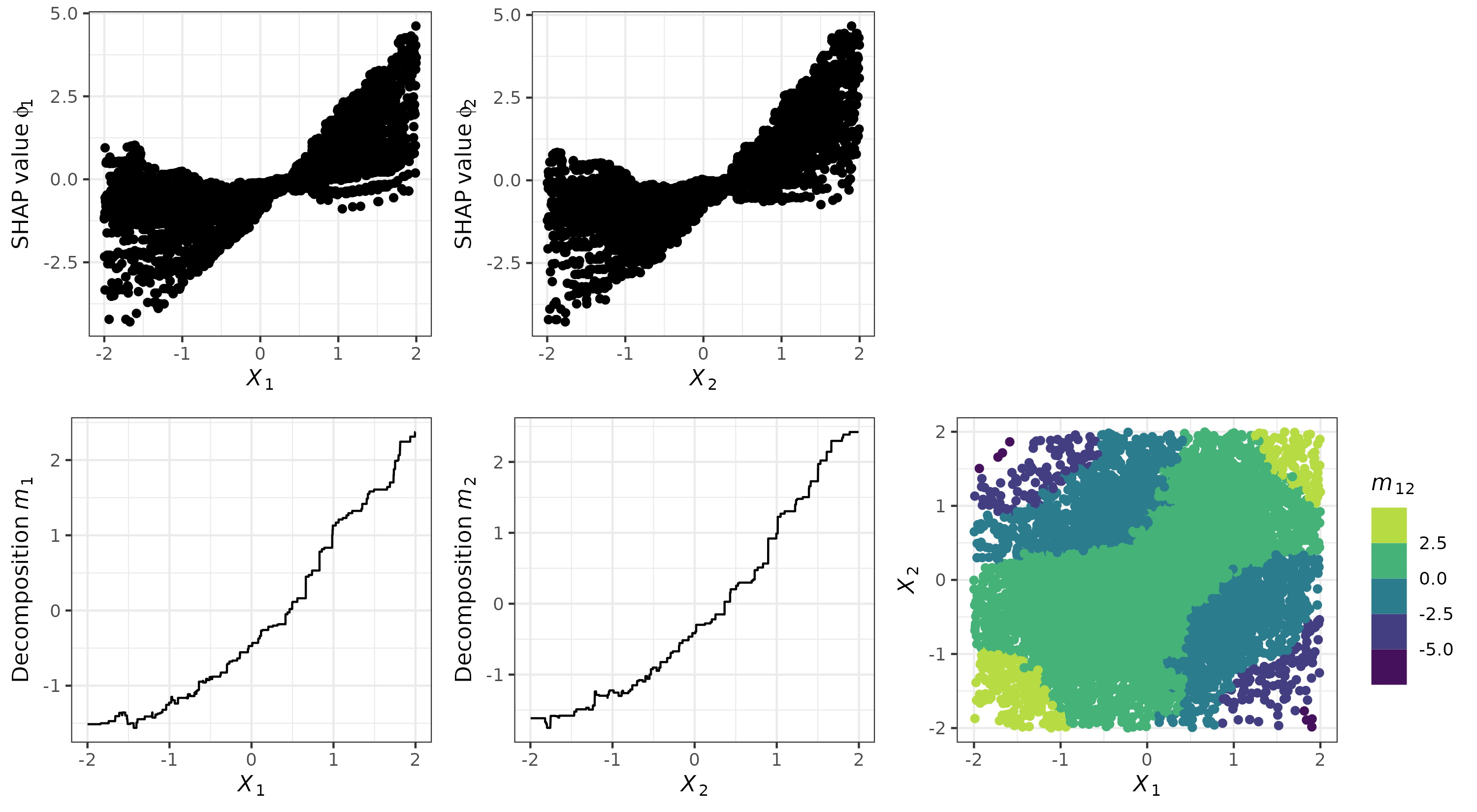} 
    \caption{Simple example. Given   an \textit{xgboost} estimator $\hat m$  estimating the true function $m(x_1,x_2)=x_1+x_2 + 2x_1x_2$, we calculate SHAP values (top row) and functional decomposition (bottom row).} 
    \label{fig:example}
\end{figure}

Those familiar with SHAP values may argue that one can detect the non-zero impact of $x_1$ by plotting $\phi_1$ over all instances (see Figure~\ref{fig:example}). This argument has two problems.
Firstly, this does not change the misleading local value. Secondly, SHAP values can be quite arbitrary:
If two estimators $m$ and $\tilde m$ are equal on the support $\mathrm{supp}(X_1,X_2)$, the corresponding SHAP values at $x\in\mathrm{supp}(X_1,X_2)$ are generally not equal. This is because SHAP values are constructed by extrapolating outside the support of the data. \cite{slack2020fooling} has empirically demonstrated how this phenomenon can be exploited to hide the importance of protected features.
One could ask for local explanations that do not extrapolate, hoping that this solves the problem. Unfortunately, this is not possible: If explanations are deduced only from the region with data support, those explanations are based on the correlation structure of the features \citep{janzing2020feature}. In particular a feature that has zero effect on the model output can still be assigned a value stemming from a correlated feature \citep{janzing2020feature,sundararajan2020many}. 
We conclude: 
\begin{quote}
    \emph{Local  explanations that do not explicitly specify all interactions cannot lead to meaningful interpretations in the presence of correlated features.}
\end{quote}
This is important to remember, noting that interpretation tools are usually used for black-box algorithms with the main purpose being to explain the model well in cases where interactions are present.
Intuitively speaking:
\begin{quote}
    \emph{
A local interpretation that explicitly considers all interactions is a global interpretation.}
\end{quote} 
Hence the goal of this paper is to unify local and global explanations.
We emphasize again that in contrast to simple SHAP or $l$-interaction SHAP $(l<q)$,  $q$-interaction SHAP provides a global explanation of a trained model.

\subsection{Related Work}

A functional decomposition for global interpretation of regression functions was introduced in the statistical literature in \cite{stone1994use}, and has been further discussed in \cite{hooker2007generalized,chastaing2012generalized,lengerich2020purifying}. These authors considered a different constraint called (generalized) functional ANOVA decomposition. In contrast, the constraint we introduce in this paper is linked to Shapley values. There is considerable literature on interactions and Shapley values. In cooperative game theory, pairwise player-player interactions were first considered by \cite{owen1972multilinear} and later generalized to higher-order interactions by \cite{grabisch1999axiomatic}. In the machine learning context, arbitrariness of Shapley values due to interactions and correlations has been discussed in
\cite{kumar2020problems}, \cite{slack2020fooling}, \cite{sundararajan2020many},  and possible solutions have been proposed in \cite{zhang2020interpreting}, \cite{kumar2021shapley}, \cite{harris2022joint}, \cite{ittner2021feature}, \cite{sundararajan2020shapley}. Recently, \cite{tsai2022faith} introduced interaction SHAP for any given order and proposed an approximation to calculate them. In this paper, we introduce an identification constraint for a functional decomposition which connects to partial dependent plots \citep{friedman2001greedy} and Shapley values with a value function that has recently been coined interventional SHAP \citep{chen2020true}.  Alternative value functions have been discussed in \cite{frye2020asymmetric}, \cite{yeh2022threading}. There are a variety of methods to obtain single-value global feature importance measures implied by SHAP. These include \cite{casalicchio2018visualizing}, \cite{frye2020asymmetric} and \cite{williamson2020efficient}, among others. Similar to our method suggested in Section \ref{sec:vim}, these measures are weighted averages of local importance values. However, in contrast to our suggestion, most are motivated by additive importance measures \citep{covert2020understanding}. 

Lastly, after the completion and upload of a first preprint of this paper, two related and relevant works were published.
\cite{bordt2022shapley} show that for every SHAP-value function there exists a one-to-one correspondence between SHAP values and  an identification in a  functional decomposition. But they do not provide explicit solutions.
 \cite{herren2022statistical}  describe an identification constraint that connects to observational SHAP.
Our present paper finds an identification constraint that gives a one-to-one correspondence to interventional SHAP and provides a fast implementation for tree-based methods.

\section{MAIN RESULTS}
	Let $(Y_i, X_{i,1},\dots,X_{i,d})$ be a data set of $n$ i.i.d. observations with $X_{i,k}\in\mathbb{R}$, $i=1,\dots,n$; $k=1,\dots,d$. We consider the supervised learning setting
	\[
	E[Y_i|X_i=x] = m(x),
	\]
where the function $m$ is of interest and $Y$ is a real valued random variable.\footnote{We use $Y_i \in \mathbb{R}$ for notational  convenience. It is straight-forward to extend to binary classification, whereas multiclass classification would require a slightly different procedure.} We assume that a reasonable estimator $\hat m$ of $m$ has been provided. 

\subsection{Global Interpretation}
With increasing dimension it can quickly get very hard, if not impossible,  to visualize and thereby comprehend a multivariate function.
Hence, a global interpretation of $\hat m$ is arguably only feasible if it is a composition of low-dimensional structures.
Let us consider a specific  decomposition of a multivariate function into a sum of main effects, bivariate interactions, etc., up to a $d$-variate interaction term.
\begin{align}\label{anova}
	\hat m(x)&= \hat m_0+\sum_{k=1}^d \hat m_k(x_{k}) \notag \\ &\quad + \sum_{k<l} \hat m_{kl}(x_{k},x_{l}) +\cdots
 + \hat m_{1,\dots,d}(x) \notag  \\&= \sum_{S\subseteq \{1,\dots,d\}} \hat m_S(x_S).
	\end{align}
The heuristic of the decomposition is that if the underlying function $m(x)$ only lives on low-dimensional structures, then  $m_S$ should be zero for most feature subsets $S$ and
the order of maximal interaction $q=\max\{ \abs{S}: m_S\neq 0\}$ should be much smaller than the number of features: $q << d$.  This discussion, however, is not very meaningful before one has agreed on an identification; without suitable identification constraints, it is possible to change components on the right without altering the left hand side.
We propose the following identification which we see as reasonable in its own right, but also connects interventional SHAP values \citep{chen2020true, janzing2020feature},
partial dependence plots \citep{friedman2001greedy} and de-biasing, as will be explained in the next three subsections.


\textbf{Marginal identification:} For every $S\subseteq \{1,\dots,d\}$,
\begin{align}\label{constraint1}
\sum_{T \cap S \neq \emptyset} \int \hat m_T(x_T) \hat p_S(x_S)\mathrm dx_S=0,
\end{align}
where $\hat p_S$ is some estimator of the density $p_S$ of $X_S$.

\begin{lemma}\label{constraint2}
The marginal identification \eqref{constraint1} is equivalent to
\begin{align*}
\sum_T \int  \hat m_T(x_T)\hat p_S(x_S)\mathrm dx_S &=\sum_{T \cap S = \emptyset} \hat m_T(x_T).
\end{align*}
for every $S\subseteq \{1,\dots,d\}$.
\end{lemma}

The next theorem states existence and uniqueness of a decomposition that satisfies the  identification constraint \eqref{constraint1} and describes the solution explicitly.

\begin{theo} \label{lem2}
Given any initial estimator $\hat m^{(0)}=\{\hat m^{(0)}_S | S\subseteq \{1,\dots,d\}\}$, 
there exists exactly one set of functions $\hat m^\ast=\{\hat m^\ast_S | S\subseteq \{1,\dots,d\}\}$ satisfying constraint \eqref{constraint1} with $\hat{m}_{\mathrm{Sum}}(x):=\sum_S \hat m^{(0)}_S(x_S) = \sum_S \hat m^\ast_S(x_S)$.
The functions are given by
\begin{align*}
\hat m^\ast_S(x_S)
&=
\sum_{V \subseteq S} (-1)^{\abs{S\setminus V}}\int \hat{m}_{\mathrm{Sum}}(x) \hat p_{-V}(x_{-V})\mathrm dx_{-V}.
\end{align*}
In particular $\hat m_S^\ast$ does not depend on the particular identification of $\hat m^{(0)}$.
\end{theo}

\begin{remark}
Theorem \ref{lem2}, given Lemma \ref{constraint2}, is a special case of a more general result in combinatorics  where the solution is known as Möbius inverse \citep{rota1964foundations}. In Cooperative game theory the result is also known as Harsanyi dividend  \citep{harsanyi1982simplified}.
\end{remark}

The next corollary provides a practically more useful identification of the solution than Theorem \ref{lem2}.

\begin{corollar}\label{thm1}
The solution $\hat m^\ast_S$ described in Theorem \ref{lem2} can be re-written as
\begin{align}\label{sol}
\hat m^\ast_S (x_S)=  \sum_{T \supseteq S} \sum_{T\setminus S \subseteq U\subseteq  T}&(-1)^{\abs{S}-\abs{T\setminus U}} \\  &  \times \int \hat m^{(0)}_T(x_T) \hat p_U(x_U)\mathrm dx_U. \notag
\end{align}
\end{corollar}

\begin{remark}
    In principle, Theorem \ref{lem2} is a special case of Corollary \ref{thm1}, where we  consider $\hat{m}^{(0)}_{\{1,\dots,d\}}=\hat{m}_{\mathrm{Sum}}$ with $\hat{m}^{(0)}_S=0$ for $S\neq \{1,\dots,d\}$ as an initial estimator. Theorem \ref{lem2} has a simpler notation. The following example shows the benefits of the Corollary: Assume $\hat{m}$ is generated by xgboost with 2 trees with depth 2. For simplicity, the first tree is split in dimensions $\{1,2\}$, the second in $\{3,4,5\}$. We write $\hat{m}(x)=\hat{m}_{t_1}(x)+\hat{m}_{t_2}(x)$, where $\hat{m}_{t_i}$ comes from tree $i$. We can use Corollary \ref{thm1} by choosing either (a) $\hat{m}^{(0)}_{1,\dots,5}(x)=\hat{m}(x)$ or (b) $\hat{m}^{(0)}_{1,2}(x)=\hat{m}_{t_1}(x), \hat{m}^{(0)}_{3,4,5}(x)=\hat{m}_{t_2}(x)$. The resulting $\hat{m}^*$ is the same since it does not depend on the specific choice of $\hat{m}^{(0)}$ as long as $\hat{m}(x)=\sum_{S}m_S^{(0)}$. Choice (a) has 3 disadvantages compared to choice (b): First, with (a) we obtain $\hat{m}^*$ by calculating $2^5-1=31$ marginals. With (b) we only need to calculate $2^2-1+2^3-1=10$. Second, with (a) we must construct $\hat{p}_S$ with $|S|$ taking values up to 5. With (b) we only require $\hat{p}_S$ for $|S|\leq 3$. Third, with (a) we need to calculate up to 5-dimensional integrals. With (b), only up to 3-dimensional integrals must be calculated. For xgboost and random planted forest, the estimator $\hat{m}$ is the sum of low dimensional functions. For estimators without this property, calculating the entire decomposition is computationally infeasible for high dimensional input.
\end{remark}

\begin{example}
Consider the setting of our simple example (Subsection \ref{sec:example}), $m(x_1,x_2)=x_1+x_2 + 2x_1x_2 $, with  $X_1$ and $X_2$ having mean
zero and variance one. If $m(x)=m^\ast(x)$ and 
$m^\ast(x)$ satisfies the population version of the marginal identification \eqref{constraint1}, then 
\begin{align*}
    m^\ast(x_1,x_2)=m^\ast_0 + m_1^\ast(x_1) + m_2^\ast(x_2) + m_{12}^\ast(x_1,x_2),
\end{align*}
with
\begin{align*}
    m^\ast_0&=2corr(X_1,X_2) \\
    m_1^\ast(x_1)&= x_1 -2corr(X_1,X_2) \\
    m_2^\ast(x_2)&=x_2 - 2corr(X_1,X_2) \\
    m_{12}^\ast(x_1,x_2)&= 2x_1x_2 + 2corr(X_1,X_2).
\end{align*}
\end{example}

\subsection{Describing Interventional SHAP in Terms of our Global Explanation}

We now show that there is a direct connection between the global explanation described in the previous section and interventional SHAP values. 
In particular, this connection describes interventional SHAP values uniquely without the use of game theoretically motivated Shapley axioms or a formula 
running through  permutations, where the number of summands grows exponentially with $d$, see formula \eqref{eq:shap2} in Appendix~\ref{sec:axioms}.
Fix a value $x_0 \in \mathbb R^d$. A local approximation at $x_0$ of the function $\hat m$ is given by
\begin{align}\label{eq:shap}
\hat m\left(x_0\right)=\phi_{0}+\sum_{k=1}^{d} \phi_k(x_0),
\end{align}
for constants  $\phi_0,\phi_1(x_0),\dots,\phi_d(x_0)$. Similar to the case of global explanations, the right hand-side is not identified.
Local explanations add constraints to equation \eqref{eq:shap} such that $\phi_k(x_0)$ is uniquely identified and best reflects the local contribution of feature $k$ to $\hat m\left(x_0\right)$.
Note that the explanation is not global because the explanation for feature $k$ depends on the value of all features $x_0=x_{0,1},\dots,x_{0,d}$ and not on $x_{0,k}$ only.

%
%
%
%
%
%


%

In what follows, we  consider the identification leading to interventional SHAP values \citep{chen2020true, janzing2020feature}, i.e., Shapley values with  value function 
\begin{align}\label{value}
v_{x_0}(S)=
\int \hat m(x)  \hat p_{-S}(x_{-S}) d x_{-S}\rvert_{x=x_0}.
\end{align}
See Appendix~\ref{sec:shapley} for a definition of Shapley values.

\begin{corollar}\label{cor:mobius}
If $\hat m$ is decomposed such that \eqref{constraint1} is fulfilled, then the interventional SHAP values are weighted averages of the corresponding components, where an interaction component is equally split to all involved features:
\[
\phi_k(x)= \hat m^{\ast}_k(x_k)+ \frac 1 2 \sum_j  \hat m^{\ast}_{kj}(x_{kj}) + \cdots + \frac 1 d \hat m^{\ast}_{1,\dots,d}(x_{1,\dots,d}).
\]
\end{corollar}

\begin{remark}
A crucial point of Corollary~\ref{cor:mobius} is that the local SHAP values can be described by the components of a global explanation. 
The result is also intriguing since usually the contribution or importance of a single feature in a general global representation as in \eqref{anova} is a complicated interplay between various interactions, see Appendix~\ref{generalexpansion}. 
\end{remark}

\subsection{Describing Partial Dependence Plots in Terms of our Global Explanation}
Given an estimator $\hat m$ and a target subset $S\subset \{1,\dots,d\}$, the partial dependence plot \citep{friedman2001greedy}, $\xi_S$, is defined as
\[
\xi_S(x_S)= \int \hat m(x) \hat p_{-S}(x_{-S})\mathrm dx_{-S}.
\]
It is straight forward to verify that partial dependence plots are linked to a functional decomposition $\{\hat m^\ast_S\}$ identified via \eqref{constraint1} through
\[
\xi_S=\sum_{U \subseteq S} \hat m^{\ast}_U.
\]
In particular if $S$ is only one feature, i.e.,  $S=\{k\}$, we have
\[
\xi_k(x_k)= \hat m_0^{\ast} + \hat m_k^\ast(x_k).
\]

\subsection{A De-Biasing Application Stemming from our Global Explanation}\label{A causal application stemming from our global explanation}

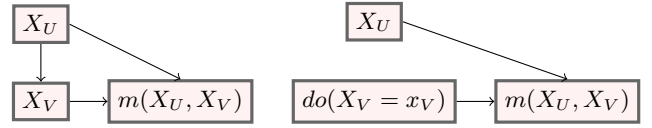
\begin{figure}
\begin{minipage}[t]{0.44\columnwidth}
\begin{tikzpicture}[
squarednode/.style={rectangle, draw=black!60, fill=red!5, very thick, minimum size=1mm,font=\small},node distance=5mm
]
\node[squarednode]      (U)                              {$X_U$};
\node[squarednode]      (V)       [below=of U] {$X_V$};
\node[squarednode]      (m)       [right=of V] {$m(X_U,X_V)$};

\draw[->] (U.south) -- (V.north);
\draw[->] (V.east) -- (m.west);
\draw[->] (U.east) -- (m.north);
\end{tikzpicture}
\end{minipage}
\begin{minipage}[t]{0.44\columnwidth}
\begin{tikzpicture}[
squarednode/.style={rectangle, draw=black!60, fill=red!5, very thick, minimum size=1mm,font=\small},node distance=5mm
]
\node[squarednode]      (U)                              {$X_U$};
\node[squarednode]      (V)       [below=of U] {$do(X_V=x_V)$};
\node[squarednode]      (m)       [right=of V] {$m(X_U,X_V)$};

\draw[->] (V.east) -- (m.west);
\draw[->] (U.east) -- (m.north);
\end{tikzpicture}
\end{minipage}
\caption{Left: Initial causal structure. Right: Causal structure after removing effect of  $X_U$ on $X_V$.}
\label{fig:graph}
\end{figure}

Assume $U$ is a set of features that should not have an effect on $\hat m$. For example 
$U=\{\text{gender, ethnicity}\}$ in the case of non-discriminatory regulation requirements. 
Assume $\{1,\dots, d\}$ is the disjoint union of $U$ and $V$ with a directed acyclic graph structure $X_U\rightarrow  X_V \rightarrow m$, $X_U\rightarrow m$; as illustrated in  Figure \ref{fig:graph}.
Eliminating the causal relationship between $X_V$ and $X_U$ can be achieved via the do-operator, $do(X_V=x_V)$, that removes all edges going into $X_V$, see Figure \ref{fig:graph}. The  function $E[m(X) |\  do(X_V=x_V))$ does not use information contained in $X_U$; neither directly nor indirectly; see also e.g. \cite{kusner2017counterfactual,lindholm2022discrimination}.
Under the assumed causal structure, standard calculations \citep{pearl2009causality} lead to
\[
E[m(X) |\ do(X_V=x_V)]= \int m(x) p_U(x_U) dx_U.
\]
 If $\hat m$ is identified via \eqref{constraint1}, then  the de-biased version $\tilde  m(x_{-U}):=\int \hat m(x) \hat p_U(x_U) dx_U$ can be extracted from $\hat m$ by dropping all components that include features in $U$:
\begin{align}
    \tilde  m(x_{-U})= \int \hat m(x) \hat p_U(x_U) dx_U=  \sum_{S \subseteq V} \hat m_S(x_S).\label{eq:feature removal}
\end{align}

\subsection{Feature Importance}\label{sec:vim}
The global interpretation also provides a new perspective on feature importance.
SHAP value feature importance for feature $k$ is usually given by an empirical version of
$E[\abs{\phi_k(x)}]$. By Corollary  \ref{cor:mobius}, 
\begin{align*}
E[\abs{\phi_k(x)}]=E\left[
\abs{\sum_{S: k \in S} \frac 1 {\abs S} \hat m^\ast_S(x_S)} \right].
\end{align*}

In this definition, contributions from various interactions and main effects can cancel each other out, which may not be desirable. An alternative is to consider
\begin{align*}
 E \left[
\sum_{S: k \in S} \frac 1 {\abs S} \abs{ \hat m^\ast_S(x_S)}, \right]
\end{align*}
or to extend the definition of feature importance to interactions by defining feature importance as 
 $E \left[\abs{ m^\ast_S(x_S)}\right]$,
for a set $S \subseteq \{1,\dots, d\}$.
\begin{example}
Going back to our simple example (Section \ref{sec:example}), where $m(x_1,x_2)=x_1+x_2 + 2x_1x_2 $, SHAP feature importance for feature $x_1$ is an empirical version of
\begin{align*}
&E\left[ \abs{X_1 - 2\textrm{corr}(X_1,X_2)- \frac 1 2 \{2X_1X_2 + 2\textrm{corr}(X_1,X_2)\}}\right] \\
&=E\left[ \abs{X_1 - X_1X_2 - \textrm{corr}(X_1,X_2) \}}\right],
\end{align*}
which merges main effect and interaction effect. Alternatively, one may consider
\begin{align*}
E&\left[ \abs{X_1 - 2\textrm{corr}(X_1,X_2)}+ \abs{  X_1X_2 + \textrm{corr}(X_1,X_2)\}}\right].
\end{align*}

\end{example}

\subsection{Remarks on our Main Results}

The marginal identification \eqref{constraint1} is new and the only identification that corresponds to partial dependence plots. Further, we are not aware of any other decompositions used in practice which allow for debiasing via post-hoc feature removal given the causal structure assumed in Subsection~\ref{A causal application stemming from our global explanation}.  Lastly, while the marginal identification corresponds to interventional SHAP,  it is possible to construct other identifications that correspond to Shapley values with other value functions; e.g. a functional ANOVA constraint leads to observational SHAP \citep{herren2022statistical}.
To the best of our knowledge, we are the first to construct (and provide open-source) an algorithm which provides a global explanation for xgboost by including all higher-order interactions.
In particular, all current implementations we are aware of only consider max. 2-way interactions (both in the case of interventional SHAP and observational SHAP). 

\section{EXPERIMENTS}\label{Experiments}
We apply our method to several real and simulated datasets to show that the functional decomposition provides additional insights compared to SHAP values and SHAP interaction values. First, we show on real data that a global explanation can provide a more comprehensive picture than a local explanation method. Second, we show on real and simulated data that the same holds for the feature importance measure proposed in Section~\ref{sec:vim}. Finally, we show that the functional decomposition allows post-hoc removal of features from a model, which can be used to reduce bias of prediction models. We performed all experiments with \textit{xgboost} and \textit{random planted forests}. The results with \textit{xgboost} are presented in Sections~\ref{sec:bike}-\ref{sec:debias} in the main paper, whereas the results with \textit{random planted forests} are in Sections~\ref{sec:bike_rpf}-\ref{sec:debias_rpf} in the Appendix.

\subsection{Global Explanations}\label{sec:bike}
As an example of a real data application, we apply our method to the \textit{bike sharing} data \citep{fanaee2014event}, predicting the number of rented bicycles per day, given seasonal and weather information. Figure~\ref{fig:bike} shows SHAP values, main effects, 2-way interactions and 3-way interactions of the features \textit{hour of the day} (hr, 0-24 full hours), \textit{Temperature} (temp, normalized to 0-1) and \textit{working day} (workingday, 0=no, 1=yes).

\begin{figure*}[ht]
    \centering
    \includegraphics[width=.8\linewidth]{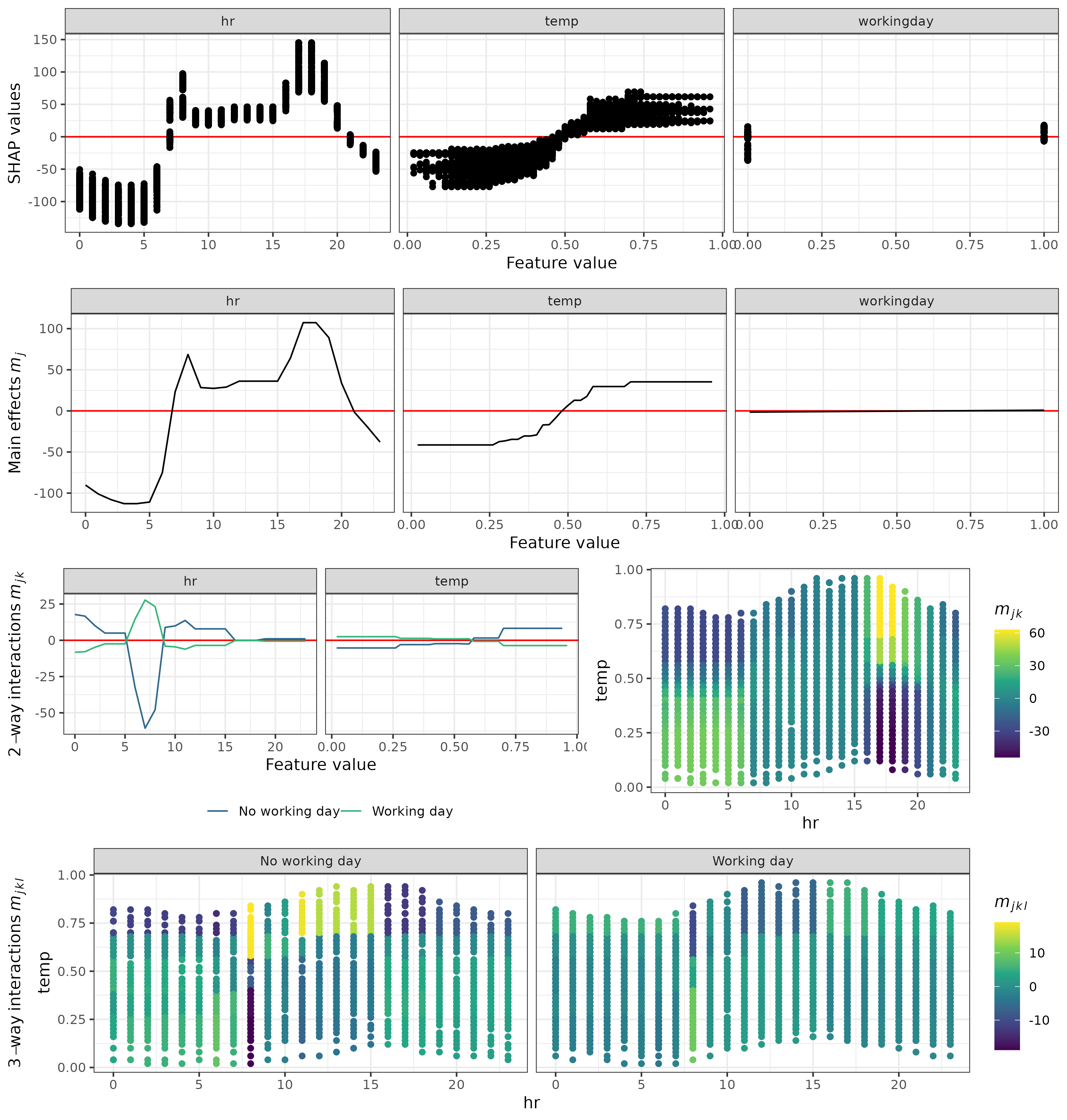}
    \caption{Bike sharing example (\textit{xgboost}). SHAP values (top row),  main effects (second row), 2-way interactions (third row) and 3-way interactions (bottom row) of the features \textit{hour of the day} (hr, 0-24 full hours), \textit{Temperature} (temp, normalized to 0-1) and \textit{working day} (workingday, 0=no, 1=yes) of the bike sharing data.} 
    \label{fig:bike}
\end{figure*}

In the top row, we see that different SHAP values are observed for the same values of the features and conclude that SHAP values are not sufficient to describe the features' effects on the outcome, due to interactions. In the second row, the main effects from the decomposition show a strong effect of the hour of the day: Many bikes are rented in the typical commute times in the morning and afternoon. We also see a positive effect of the temperature and no main effect of whether or not it is a working day. The 2-way interactions in the third row reveal strong interactions between the hour of the day and working day: On working days, more bikes are rented in the morning and less during the night and around noon. We also see that the temperature has a slightly higher effect on non working days and in the afternoon. In the bottom row, the 3-way interactions show that interactions between the hour of the day and the temperature are stronger on non working days than on working days. 

We conclude that the full functional decomposition provides a more comprehensive picture of the features' effects, compared to usual SHAP value interpretations and 2-way interaction SHAP, as e.g. proposed by \cite{lundberg2020local}. Note that, as described above, our methods do indeed provide the full picture, including all higher-order interactions, whereas Figure~\ref{fig:bike} only shows a subset of these interactions.  

\subsection{Feature Importance}\label{sec:expvim}
As described in Section~\ref{sec:vim}, the functional decomposition can also be used to calculate feature importance. Figure~\ref{fig:vim} shows the feature importance for the function $m(x) = x_1 + x_3 + x_2 x_3 - 2 x_2 x_3 x_4$ and the bike sharing data from Section~\ref{sec:bike} based on SHAP values and our functional decomposition. For the simple function, the SHAP feature importance identifies $x_1$ and $x_3$ as equally important and $x_2$ and $x_4$ as less important but it gives no information about interactions. On the other hand, the feature importance based on the functional decomposition shows that $x_1$ has a strong main effect but no interactions, whereas $x_2$ and $x_4$ have only interaction effects but no main effects and $x_3$ both kinds of effects. Similarly, on the bike sharing data, the hour of the day (feature \textit{hr}) and the temperature (\textit{temp}) have both main and interaction effects, whereas the feature \textit{working day} has 2-way interaction effects but no main effects (compare Figure~\ref{fig:bike}). Note that both definitions of feature importance are based on absolute values of SHAP values or components $m_S$ and thus are non-negative, in contrast to other methods of feature importance \citep{nembrini2018revival,casalicchio2018visualizing}. 

\begin{figure}[t]
    \centering
    \includegraphics[width=\linewidth]{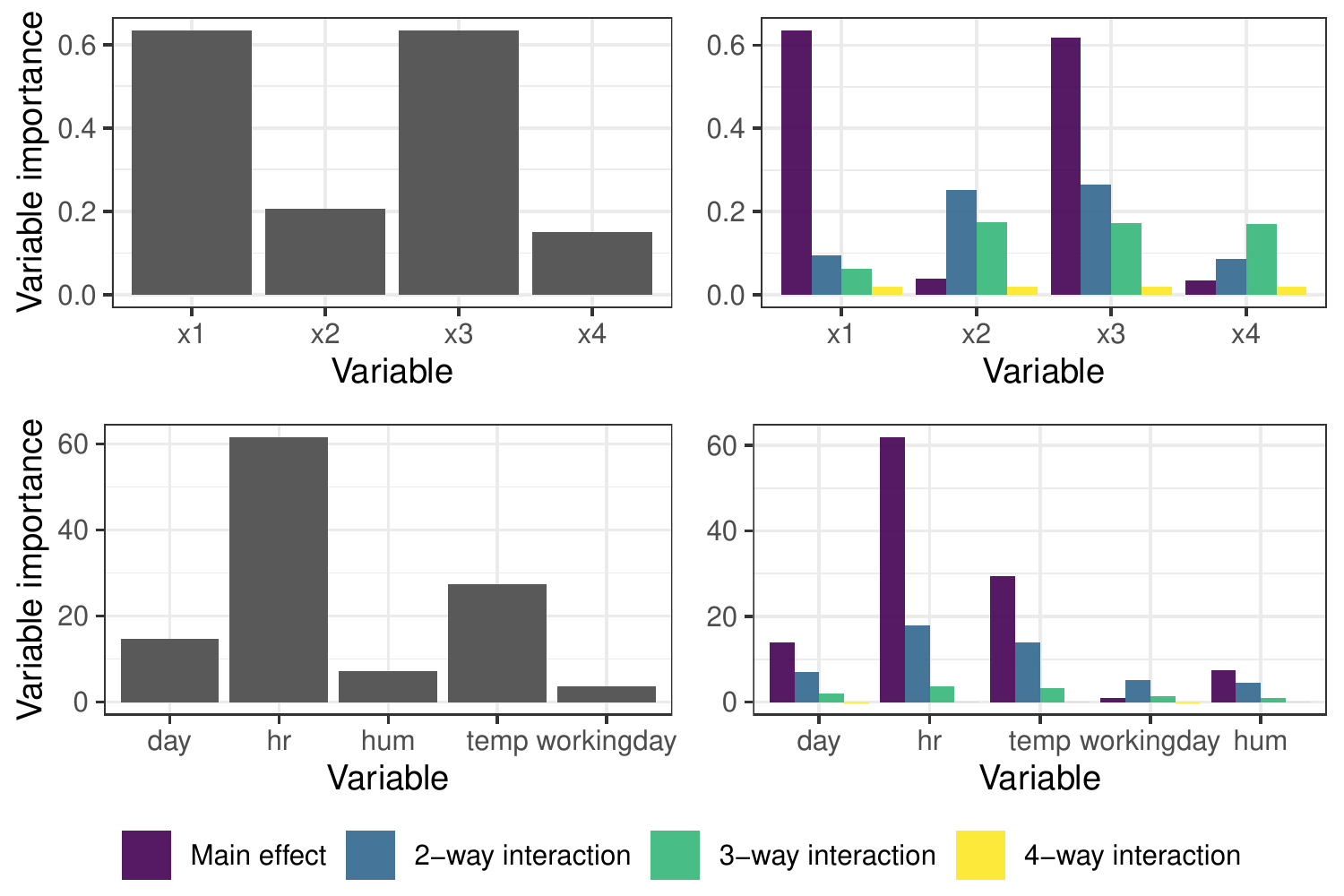}
    \caption{Feature importance (\textit{xgboost}) for the function $m(x) = x_1 + x_3 + x_2 x_3 - 2 x_2 x_3 x_4$ (top row) and the bike sharing data from Section~\ref{sec:bike} (bottom row)  based on SHAP values (left column) and our functional decomposition separately for main effects and interactions of different orders (right column). } 
    \label{fig:vim}
\end{figure}

\subsection{Post-hoc Feature Removal}\label{sec:debias}
We show that our method can be used to remove features and all their effects, including interactions, from a model \textit{post-hoc}, i.e., after model fitting. The idea is that in the setting of Subsection \ref{A causal application stemming from our global explanation}, using our decomposition, feature removal can be done by setting all components $m^*_S=0$ for all $S\subseteq\{1,\dots,d\}$ which include at least one component to be removed as in \eqref{eq:feature removal}.

We trained models on simulated data and the \textit{adult} dataset \citep{Dua:2019}. Both models contained a feature \textit{sex} or \textit{gender}, which is a protected attribute and should not have an effect in fair prediction models \citep{barocas-hardt-narayanan}. In the simulation, we considered the simplified scenario where we predict a person's salary, based on their sex and weekly working hours. We set the weekly working hours to an average of 40 for men and to 30 for women. Salary was simulated as 1 unit (e.g. thousand Euro per year) per weekly working hour and an additional 20 for males (see Figure~\ref{fig:graph}). Thus, men earn more for working longer hours (on average) and for being male per se. The first effect should be kept by a fair machine learning model, whereas the second effect is discriminating women. In the \textit{adult} data, we have the same features \textit{sex} and \textit{hours} but we do not know the causal structure. 

Figrue~\ref{fig:dediscr} shows the prediction for females and males of the full model, a refitted model without the protected feature \textit{sex} and a decomposed model where the feature \textit{sex} was removed post-hoc. In the simulated data, we see that refitting the model does not change the predictions at all: Because of the high correlation between \textit{sex} and \textit{hours}, the effects of \textit{sex} cannot be removed by not considering the feature in the model. Our decomposition, on the other hand, allows us to remove the (unwanted) direct effect of \textit{sex} while keeping the (wanted) indirect effect through \textit{hours}. On the \textit{adult} data, we see a similar difference, but less pronounced.

\begin{figure}[t]
    \centering
    \includegraphics[width=\linewidth]{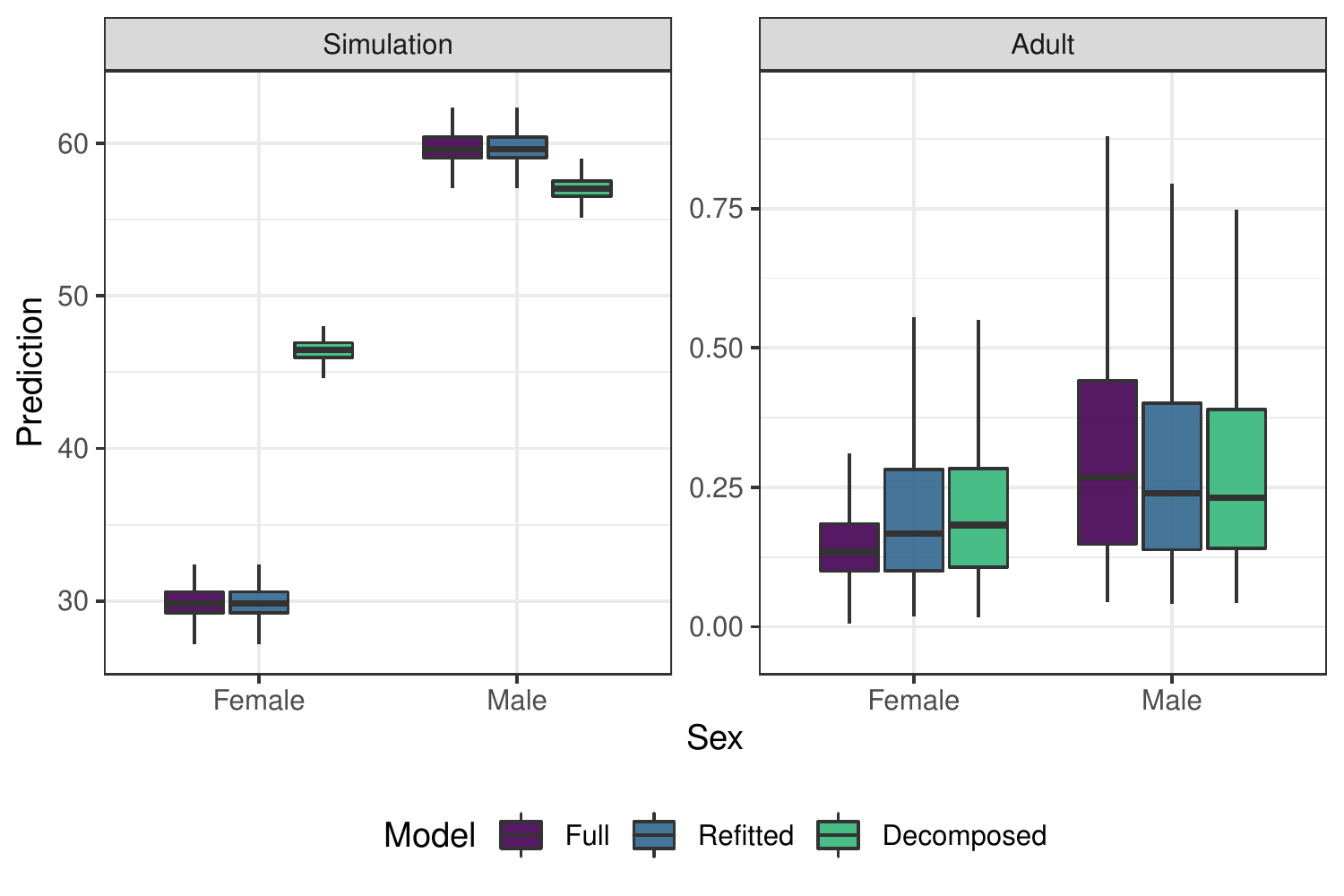}
    \begin{tabular}{lrrr}
    \toprule
    Setting & \multicolumn{3}{c}{Median difference} \\ 
    & Full & Refitted & Decomposed \\
    \midrule
    Simulation & 29.79 & 29.79 & 10.57 \\
    Adult & 0.13 & 0.07 & 0.05 \\ 
    \bottomrule
    \end{tabular}   
    \caption{Post-hoc feature removal (\textit{xgboost}). Predictions in a simulation (left) and the \textit{adult} dataset for males and females of the full model, a refitted model without the protected feature \textit{sex} and a decomposed model where the feature \textit{sex} was removed post-hoc. The table below shows the median differences between females and males for the three models.}
    \label{fig:dediscr}
\end{figure}

\section{CONCLUDING REMARKS AND LIMITATIONS}
In this paper, we have introduced a way to turn local Shapley values into a global explanation using a functional decomposition.
The explanation has a de-biasing application under the DAG structure given in Figure \ref{fig:graph}. This causal structure might be quite realistic in many fairness considerations, but the true causal structure is generally unknown. In this respect, it would be interesting to look into other causal structures that could motivate different identification constraints, which may connect to other local explanations than interventional SHAP.
Indeed, \cite{bordt2022shapley} show that every SHAP-value function corresponds to a specific identification.
Also, while our suggestions for feature importance measures paint a more precise picture in many cases, it is not directly motivated by a theoretical constraint, such as usual additive importance measures. It will require more research to back these ideas by theory. Another point not considered in this paper is the difference between the estimate $\hat m$ and a potential true function $m$. In particular, it is not clear if a method that estimates $m$ well is also a good estimator for a selection of components {$m_S$}. This discussion is related to work done in double/debiased machine learning \citep{chernozhukov2018double} and more general in semiparametric statistics, see e.g. \cite{bickel1993efficient}. Moving forward, it could be interesting to modify out of the box machine learning algorithms to specifically learn the low dimensional structures well.

\paragraph{Ethical implications} Generally, explaining prediction models can help to reduce bias or discrimination. Specifically, our methods can be used to reveal higher-order interactions with protected attributes and by that detect bias and reduce such bias by post-hoc feature removal (see Section~\ref{sec:debias}). However, there is more to \textit{fair machine learning} than removing effects of protected attributes \citep[see e.g.][]{barocas-hardt-narayanan} and, as shown by \cite{slack2020fooling}, machine learning explanation methods are not immune to (adversarial) attacks. Thus, results should be interpreted with care. 

\subsection*{Acknowledgments}
 MNW  received funding from the German Research Foundation (DFG), Emmy Noether Grant 437611051. MH has carried out this research in association with the project framework InterAct. JTM was funded by the German Research Foundation (DFG) through the Research Training Group RTG 1953.


\bibliography{mybib}



\clearpage
\onecolumn
\appendix
\section{SHAPLEY VALUES} \label{sec:shapley}
Consider a value function $v_{x_0}$ that assigns a real value
$v_{x_0}(S)$ to each subset $S \subseteq \{1,\dots d\}$.  Shapley axioms provide a unique solution under the four axioms
efficiency, symmetry, dummy and additivity \citep{shapley1953value}, see Appendix~\ref{sec:axioms}.
Defining
$
\Delta_{v}(k, S)=v(S \cup k)-v(S)$,
the Shapley values are
\begin{align}\label{eq:shap2}
\phi_{k}&=\frac{1}{d !} \sum_{\pi \in \Pi_d} \Delta_{v}\left(k, \{\pi(1),\dots,\pi(k-1)\}\right)\\
&=\frac 1 {d!}\sum_{S \subseteq \{1,\dots,d\} \setminus\{k\}} {|S| !(d -|S|-1) !}\Delta_{v}(k, S),
\end{align}
where $\Pi_d$ is the set of permutations of $\{1,\dots,d\}$. We follow \cite{janzing2020feature} and define interventional SHAP values as Shapley values with the value function 
\begin{align}\label{value}
v_{x_0}(S)=
\int \hat m(x)  \hat p_{-S}(x_{-S}) d x_{-S}\rvert_{x=x_0},
\end{align}
which is also the version implemented in TreeSHAP \citep{lundberg2020local}.


\section{SHAPLEY AXIOMS} \label{sec:axioms}
Given a function $m$, a point $x_0$, and a value function $v$, the Shapley axioms \citep{shapley1953value} are

\begin{itemize}
\item \textbf{Efficiency}: $ m\left(x_0\right)=\phi_{0}+\sum_{k=1}^{d} \phi_k(x_0)$.
\item \textbf{Symmetry}: Fix any $k,l \in \{1,\dots,d\}, k\neq l$. 
If $v_{x_0}(S\cup k)=v_{x_0}(S\cup l)$, for all  $S \subseteq \{1,\dots d\}\setminus \{k,l\}$, then $\phi_k(x_0)=\phi_l(x_0)$
\item \textbf{Dummy}  If $v_{x_0}(S\cup k)=v_{x_0}(S)$, for all $S \subseteq \{1,\dots d\}\setminus \{k\}$, then $\phi_k=0$
\item \textbf{Linearity} If $m(x_0)=m^1(x_0)+m^2(x_0)$, then  $\phi_k(x_0)=\phi^1_k(x_0)+\phi^2_k(x_0)$, where $\phi^l$ is the explanation corresponding to the function $m^l$.
\end{itemize}
\section{LEMMATA AND PROOFS}

\subsection{Lemmata}

\begin{lemma}[\cite{shapley1953value}] \label{mobius}For every $U \subseteq\{1,\dots ,d\}$,
\[
\int \hat m(x) p_{-U}(x_{-U})\mathrm dx_{-U}= \sum_{T\subseteq U}  \hat m^{\ast}_T(x_T)
\]
\end{lemma}

\begin{proof}
\begin{align*}
\sum_{T\subseteq U} \hat m_T(x_T)
&= \sum_{T \subseteq U} \sum_{V \subseteq T} (-1)^{\abs{T\setminus V}}\int  \hat{m}_{\mathrm{Sum}}(x) p_{-V}(x_{-V})\mathrm dx_{-V}\\
&= \sum_{V \subset U} \int  \hat{m}_{\mathrm{Sum}}(x) p_{-V}(x_{-V})\mathrm dx_{-V} \sum_{S \subseteq\{1,\dots, \abs{U\setminus V}\} } (-1)^{\abs S}\\
&\quad +  \int  \hat{m}_{\mathrm{Sum}}(x) p_{-U}(x_{-U})\\
&=  \int \hat{m}_{\mathrm{Sum}}(x) p_{-U}(x_{-U}),
\end{align*}
where the last equation follows from  $\sum_{S \subseteq\{1,\dots, \abs{U-V}\} } (-1)^{\abs S}$=0, noting that a non-empty set has an equal number of subsets with an odd number of elements as subsets with an even number of elements.
\end{proof}

\subsection{Proofs}

\begin{proof}[Proof of Lemma  \ref{constraint2}]
Note that for every $S\subseteq \{1,\dots,d\}$
\begin{align*}
\sum_T \int  \hat m_T(x_T)p_S(x_S)\mathrm dx_S &=\sum_{T \cap S = \emptyset} \hat m_T(x_T)  + \sum_{T \cap S \neq \emptyset} \int \hat m_T(x_T) p_S(x_S)\mathrm dx_S.
\end{align*}
Hence, condition \eqref{constraint1} is equivalent to
\[
\sum_T \int  \hat m_T(x_T)p_S(x_S)\mathrm dx_S =\sum_{T \cap S = \emptyset} \hat m_T(x_T).
\]
\end{proof}

\begin{proof}[Proof of Theorem  \ref{lem2}]
As stated in the main text, this result is a well known result. Here we provide an alternative proof by 
showing that the solution is the same as the solution of Corollary \ref{thm1}. In the proof of Corollary \ref{thm1} we then show that the solution is the desired one.

We consider a fixed $S\subseteq \{1,\dots, d\}$. 
We will make use of the fact that for a set  $T\not\supseteq  S$ 
\begin{align}\label{subsets}
&\sum_{V \subseteq S} (-1)^{\abs{S\setminus V}}\int\hat  m^{(0)}_T(x_T) p_{-V}(x_{-V})\mathrm dx_{-V} =0,
\end{align}
and for a set $T\subseteq \{1,\dots, d\}, T\supseteq S$
\begin{align}\label{subsets2}
\{U: T\setminus S \subseteq U \subseteq T\}=\{ T \setminus V: V \subseteq S\}
\end{align}
Combining  \eqref{subsets}--\eqref{subsets2}, we get
\begin{align*}
&\sum_{V \subseteq S} (-1)^{\abs{S\setminus V}}\int \hat{m}_{\mathrm{Sum}}(x) p_{-V}(x_{-V})\mathrm dx_{-V}\\
&=\sum_{T\subseteq \{1,\dots,d\}}\sum_{V \subseteq S} (-1)^{\abs{S\setminus V}}\int \hat m^{(0)}_T(x_T) p_{-V}(x_{-V})\mathrm dx_{-V}\\
&=\sum_{T\supseteq S}\sum_{V \subseteq S} (-1)^{\abs{S}-\abs{V}}\int \hat m^{(0)}_T(x_T) p_{T\setminus V}(x_{-T\setminus V})\mathrm dx_{T\setminus V}\\
&=\sum_{T \supseteq S} \sum_{T\setminus S \subseteq U\subseteq  T}(-1)^{\abs{S}-\abs{T\setminus U}}\int \hat m^{(0)}_T(x_T) p_U(x_U)\mathrm dx_U.
\end{align*}
It is left to show  \eqref{subsets}--\eqref{subsets2}. Equation \eqref{subsets2} follows from straight forward calculations.
To see \ref{subsets}, note
\begin{align*}
&\sum_{V \subseteq S} (-1)^{\abs{S\setminus V}}\int \hat m^{(0)}_T(x_T) p_{-V}(x_{-V})\mathrm dx_{-V} \\
=&\sum_{U \subseteq S\cap T} \sum_{W \subseteq S\setminus T}(-1)^{\abs{S\setminus \{W\cup U\}}}\int \hat m^{(0)}_T(x_T) p_{-\{U\cup W\}}(x_{-\{U\cup W\}})\mathrm dx_{-\{U\cup W\}} \\
=&\sum_{U \subseteq S\cap T} \int \hat m^{(0)}_T(x_T) p_{-U}(x_{-U})\mathrm dx_{-U} \sum_{W \subseteq S\setminus T}(-1)^{\abs{S\setminus \{W\cup U\}}} \\
=&\sum_{U \subseteq S\cap T} \int \hat m^{(0)}_T(x_T) p_{-U}(x_{-U})\mathrm dx_{-U} \left(\sum_{W \subseteq S\setminus T, \abs W=\text{odd}}(-1)^{\abs{S\setminus U} - 1}+\sum_{W \subseteq S\setminus T, \abs W=\text{even}}(-1)^{\abs{S\setminus U} }\right) \\
=&0,
\end{align*}
where the last equality follows from the fact that every non-empty set has an equal number of odd and even subsets.
\end{proof}

\begin{proof}[Proof of Corollary \ref{cor:mobius}]
This proof is analogue to the proof of Lemma 3.1 in \cite{shapley1953value}.
From \ref{mobius}, We have
$
\int \hat m(x) p_{-U}(x_{-U})\mathrm dx_{-U}= \sum_{T\subseteq U} \hat m^{\ast}_T(x_T).
$
Hence the game $\hat m$ with value function
\[
v_{\hat m}(U)=\int \hat  m(x) p_{-U}(x_{-U})\mathrm dx_{-U}
\]
equals the game
$m^\ast$  with value function
\[
v_{\hat m^\ast}(U)=\sum_{S \subseteq \{1,\dots,d\}} \hat m^{\ast}_S(x_S)\delta_S(U), \quad \delta_S(U)= 1(  S\subseteq U).
\]
We now concentrate on the function $\hat m_S^\ast$ with value function 
$\hat m^{\ast}_S(x_S)\delta_S(U)$.
We will show that for every non-empty $S \subseteq\{1,\dots, d\}$,
\begin{align}\label{proof:shap}
\phi_k(x,\hat m^{\ast}_S(x_S)\delta_S(U)) = 1(k \in S) \abs S^{-1} \hat m^{\ast}_S(x_S).
\end{align}
Here, $\phi_k(x,v)$ denotes the Shapley value for feature $k$ at point $x$ in a game with value function $v$.
 The proof is then completed by the linearity axiom, together with
\[
\phi_k(x, \hat m^{\ast}_\emptyset\delta_\emptyset(U) )= \begin{cases}\hat  m^{\ast}_\emptyset, &k=0, \\  0 & \text{else}.\end{cases}
\]
The last statement follows from the dummy axiom.
To show \eqref{proof:shap}, let's assume that $j,k \in S$. Then for every 
$U\subseteq\{1,\dots, d\}$, $\hat m^{\ast}_S(x_S)\delta_S(U\cup j)=\hat m^{\ast}_S(x_S)\delta_S(U\cup k)$,
which, by the symmetry axiom, implies 
\[
\phi_j(x,\hat m^{\ast}_S(x_S)\delta_S(U)) =\phi_k(x,\hat m^{\ast}_S(x_S)\delta_S(U)).
\]
Additionally, for $k\not \in S$, we have $\phi_j(x,\hat m^{\ast}_S(x_S)\delta_S(U)) =0$, by the dummy axiom. Hence we conclude \eqref{proof:shap} by applying \eqref{eq:shap}.
\end{proof}

\begin{proof}[Proof of Corollary \ref{thm1}]
Lemma \ref{mobius} implies that $\hat m^\ast$, as defined in \eqref{sol},  is a solution of the marginal identification.
To see this, for $S\subseteq\{1,\dots ,d\}$, consider the following decomposition of $\int m^\ast(x) p_S(x_S)\mathrm dx_S$
\begin{align*}
\int \hat m^\ast(x) p_S(x_S)\mathrm dx_S&=  \sum_{T \cap S \neq \emptyset} \int \hat  m^\ast_T(x_T) p_S(x_S)\mathrm dx_S + \sum_{T \cap S = \emptyset}  \hat m^\ast_T(x_T).
\end{align*}
Using Lemma \ref{mobius}, we have 
\[
\int  \hat m^\ast(x) p_S(x_S)\mathrm dx_S=\sum_{T\subseteq S^c} \hat m^\ast_T(x_T)=\sum_{T \cap S = \emptyset} \hat m^\ast_T(x_T),
\]
which with the previous statement implies that $ \hat m^\ast$ is a solution:
\[
\sum_{T \cap S \neq \emptyset} \int   \hat m^\ast_T(x_T) p_S(x_S)\mathrm dx_S=0.
\]
It is left to show that the solution is unique. 
Assume that there are two set of functions $\hat m^{\circ}$ and $\hat m^{\ast}$ that satisfy
\eqref{constraint1}  with $\sum_S \hat m_S^{\circ}=\sum_S\hat m_S^{\ast}$. From Lemma \ref{constraint2},
it follows that for all $S\subseteq \{1,\dots,d\}$
\[
\sum_{T \cap S = \emptyset} \hat m^{\circ}_T(x_T)=\sum_{T \cap S = \emptyset} \hat m^{\ast}_T(x_T),
\]
implying $\hat m^{\circ}_T(x_T)=\hat m^{\ast}_T(x_T)$ for all $T\subseteq\{1,\dots,d\}.$
\end{proof}

\section{CONNECTING A GENERAL GLOBAL EXPANSION TO SHAP VALUES} \label{generalexpansion}
If a regression or classification function $m$ is not identified via \eqref{constraint1},
then calculating interventional SHAP values from such a decomposition leads to lengthy and non-trivial expressions. Here, we show how the terms up to dimension three in a general  non-identified decomposition
enter into a SHAP value. The following formula follows from straight forward calculations using \eqref{eq:shap}.

For $v_{x}(S)= \int m(x){p_{-S}(x_{-S})} d x_{-S}$, and $m(x)=\sum_{S\subseteq \{1,\dots,d\}} m_S(x_S)$, we get


\begin{align*}
\phi_1(x_0)&= m_1(x_1) - E[m_1(X_1)] \\
&+ \frac 1 2  \left \{\sum_{j\neq1} m_{1j}(x_1,x_j)  -  E[m_{1j}(X_1,X_j)] \right.  \\ 
& \qquad  \quad  +  \left. \sum_{j\neq 1} E[m_{1j}(x_1,X_j)] -   E[m_{1j}(X_1,x_j)] \right\}\\
&+ \frac 1 3  \left \{\sum_{j,k \neq 1, j < k} m_{1jk}(x_1,x_j,x_k)  - E[m_{1jk}(X_1,X_j, X_k)]\right.\\
& \qquad  \quad  +    \sum_{j,k \neq 1, j< k}   E[m_{1jk}(x_1,X_j,X_k)]-E[m_{1jk}(X_1,x_j, x_k)] \\
& \qquad  \quad  +     \frac 1 2\sum_{j,k \neq 1,j<k}E[m_{1jk}(x_1,X_j, x_k)]  - E[m_{1jk}(X_1,x_j,X_k)] \\
& \qquad  \quad  +    \frac 1 2 \left.\sum_{j,k \neq 1, j<k} E[m_{1jk}(x_1,x_j, X_k)]  - E[m_{1jk}(X_1,X_j,x_k)] \right\}\\
&+ \frac 1 4  \left \{\sum_{j,k,l \neq 1, j< k< l} m_{1jkl}(x_1,x_j,x_k, x_l)  - E[m_{1jkl}(X_1,X_j, X_k, X_l)]\right.\\
& \qquad  \quad  +    \sum_{j,k \neq 1,j< k< l}   E[m_{1jk}(x_1,X_j,X_k, X_l)]-E[m_{1jk}(X_1,x_j, x_k,x_l)] \\
& \qquad  \quad  +     \frac 1 2\sum_{j,k \neq 1, j< k< l}E[m_{1jk}(x_1,x_j, X_k,X_l)]  - E[m_{1jk}(X_1,X_j,x_k,x_l)] \\
& \qquad  \quad  +    \frac 1 2 \sum_{j,k \neq 1, j< k< l} E[m_{1jk}(x_1,X_j, x_k,X_l)]  - E[m_{1jk}(X_1,x_j,X_k,x_l)] \\
& \qquad  \quad  +    \frac 1 2 \sum_{j,k \neq 1, j< k< l} E[m_{1jk}(x_1,X_j, X_k,x_l)]  - E[m_{1jk}(X_1,x_j,x_k,X_l)] \\
& \qquad  \quad  +     \frac 1 3\sum_{j,k \neq 1, j< k< l}E[m_{1jk}(x_1,x_j, x_k,X_l)]  - E[m_{1jk}(X_1,X_j,X_k,x_l)] \\
& \qquad  \quad  +    \frac 1 3 \sum_{j,k \neq 1, j< k< l} E[m_{1jk}(x_1,x_j, X_k,x_l)]  - E[m_{1jk}(X_1,X_j,x_k,X_l)]\\
& \qquad  \quad  +    \frac 1 3 \left.\sum_{j,k \neq 1, j< k< l} E[m_{1jk}(x_1,X_j, x_k,x_l)]  - E[m_{1jk}(X_1,x_j,X_k,X_l)] \right\}\\
& \qquad \qquad \qquad \qquad  \cdots
\end{align*}

\section{CALCULATING THE FUNCTIONAL DECOMPOSITION OF SHAP VALUES FROM LOW-DIMENSIONAL TREE STRUCTURES} \label{sec:calculate}
Our proposed decomposition can be calculated from tree-based models by directly applying Corollary~\ref{thm1}. Inspired by \cite{lundberg2020local}, we first describe naïve algorithms for \textit{xgboost} and \textit{random planted forest} models and then describe an improved algorithm for \textit{xgboost} that only needs a single recursion through each tree.

\subsection{Naïve xgboost Algorithm}\label{sec:xgboost_naive}
For all subsets of features $S \subseteq \{1,\dots, d\}$, we calculate the decomposition $\hat{m}_S(x_i)$ for all observations of interest $x_i \in \mathbf{X}$ recursively for each tree with features $T$ by considering all subsets $U, T$ with $T\setminus S \subseteq U\subseteq  T$. In each node of a tree, if the node is a leaf node we return it's prediction (e.g. the mean in CART-like trees). For internal (non-leaf) nodes, the procedure depends on whether the feature used for splitting in the node is in the subset $U$ or not. If the feature is in the subset $U$, we continue in both the left and right children nodes, each weighted by the coverage, i.e. the proportion of training observations going left and right, respectively. If the feature is not in the subset $U$, we apply the splitting criterion of the node and continue with the respective node selected by the splitting procedure for observation $x_i$. See Algorithm~\ref{alg:naive} for the full algorithm in pseudo code. 

\begin{algorithm}[H]
  \caption{\sc Naïve xgboost algorithm}
  \label{alg:naive}
\begin{algorithmic}
    \STATE {\bfseries Procedure:} \textsc{decompose}($\mathbf{X}$, $\hat{m}(x)$)
    
    \STATE {\bfseries Input:} Dataset to be explained $\mathbf{X} \in \mathbb{R}^{n \times d}$, tree-based model $\hat{m}(x)$ with $B$ trees
    \STATE {\bfseries Output:} Components $\hat{m}_S(x_i)$ for all $S \subseteq \{1,\dots, d\}$ and $x_i \in \mathbf{X}$
    \FOR{$i \in 1,\dots,n$}
        \FOR{$S \subseteq \{1,\dots, d\}$}
            \STATE $\hat{m}_S(x_i) \gets 0$
            \FOR{$\text{tree} \in \{1,\dots, B\}$}
                \IF{$T \supseteq S$}
                    \FOR{$U: T\setminus S \subseteq U\subseteq  T$}
                        \STATE $\hat{m}_S(x_i) \gets \hat{m}_S(x_i) + (-1)^{\abs{S}-\abs{T\setminus U}} \textsc{recurse}(\text{tree}, U, x_i, 0)$
                    \ENDFOR
                \ENDIF
            \ENDFOR
        \ENDFOR
    \ENDFOR
    \STATE {\bfseries Return:} $\hat{m}_S$
    \STATE
    \STATE {\bfseries Procedure:} \textsc{recurse}($\text{tree}, U, x_i, \text{node}$)
    \STATE {\bfseries Input:} Tree ID (tree), subset $U$, data point $x_i$, node ID (node)
    \STATE {\bfseries Output:} Coverage-weighted prediction
    \IF{\textsc{isleaf}(node)}
        \STATE {\bfseries Return:} \textsc{prediction}(node)
    \ELSE
        \STATE $j \gets$ \textsc{split-feature}(node)
        \IF{$j \in U$}
            \STATE $C_{\text{left}} \gets$ \textsc{coverage}(left-node)
            \STATE $C_{\text{right}} \gets$ \textsc{coverage}(right-node)
            \STATE {\bfseries Return:} $C_{\text{left}} \textsc{recurse}(\text{tree}, U, x_i, \text{left-node}) + C_{\text{right}} \textsc{recurse}(tree, U, x_i, \text{right-node})$
        \ELSE
            \IF{$x_i^j \leq$ \textsc{split-value}(node)}
                \STATE {\bfseries Return:} $\textsc{recurse}(\text{tree}, U, x_i, \text{left-node})$
            \ELSE
                \STATE {\bfseries Return:} $\textsc{recurse}(\text{tree}, U, x_i, \text{right-node})$
            \ENDIF
    \ENDIF
  \ENDIF
\end{algorithmic}
\end{algorithm}

\subsection{Improved xgboost Algorithm}
To improve the algorithm described in Section~\ref{sec:xgboost_naive} and Algorithm~\ref{alg:naive}, we pre-calculate the contribution of each tree for all $n$ observations and tree-subsets $T$ in a single recursive procedure by filling an $n \times 2^D$ matrix, where $D$ is the tree depth. In a second step, we just have to sum these contributions with the corresponding sign (see Corollary~\ref{thm1}). See Algorithm~\ref{alg:improved} for the algorithm in pseudo code. 

\begin{algorithm}[H]
  \caption{\sc Improved xgboost algorithm}
  \label{alg:improved}
\begin{algorithmic}
    \STATE {\bfseries Procedure:} \textsc{decompose}($\mathbf{X}$, $\hat{m}(x)$)
    
    \STATE {\bfseries Input:} Dataset to be explained $\mathbf{X} \in \mathbb{R}^{n \times d}$, tree-based model $\hat{m}(x)$ with $B$ trees
    \STATE {\bfseries Output:} Components $\hat{m}_S(x_i)$ for all $S \subseteq \{1,\dots, d\}$ and $x_i \in \mathbf{X}$
    \FOR{$S \subseteq \{1,\dots, d\}$}
        \STATE $\hat{m}_S(\mathbf{X}) \gets 0$
    \ENDFOR
    \FOR{$\text{tree} \in \{1,\dots, B\}$}
        \STATE $\mathcal{U} = \{U: U\subseteq  T\}$
        \STATE $\mathbf{\hat{M}}(\mathbf{X}, \mathcal{U}) \gets  \textsc{recurse}(\text{tree}, \mathcal{U}, \mathbf{X}, 0)$
        \FOR{$S \subseteq \{1,\dots, d\}$}
            \IF{$T \supseteq S$}
                \FOR{$U: T\setminus S \subseteq U\subseteq  T$}
                    \STATE $\hat{m}_S(\mathbf{X}) \gets \hat{m}_S(\mathbf{X}) + (-1)^{\abs{S}-\abs{T\setminus U}} \mathbf{\hat{M}}(\mathbf{X}, U) $
                \ENDFOR
            \ENDIF
        \ENDFOR
    \ENDFOR
    \STATE {\bfseries Return:} $\hat{m}_S$
    \STATE
    \STATE {\bfseries Procedure:} \textsc{recurse}($\text{tree}, \mathcal{U}, \mathbf{X}, \text{node}$)
    \STATE {\bfseries Input:} Tree ID (tree), set of subsets $\mathcal{U}$, data matrix $\mathbf{X}$, node ID (node)
    \STATE {\bfseries Output:} Matrix of coverage-weighted predictions
    \IF{\textsc{isleaf}(node)}
        \STATE $\mathbf{\hat{M}}(\mathbf{X},\mathcal{U}) \gets \textsc{prediction}(node)$
    \ELSE
        \STATE $\mathbf{\hat{M}}_\text{left}(\mathbf{X},\mathcal{U}) \gets \textsc{recurse}(\text{tree}, \mathcal{U}, \mathbf{X}, \text{left-node})$
        \STATE $\mathbf{\hat{M}}_\text{right}(\mathbf{X},\mathcal{U}) \gets \textsc{recurse}(\text{tree}, \mathcal{U}, \mathbf{X}, \text{right-node})$
        \STATE $C_{\text{left}} \gets$ \textsc{coverage}(left-node)
        \STATE $C_{\text{right}} \gets$ \textsc{coverage}(right-node)
        \STATE $j \gets$ \textsc{split-feature}(node)
        \STATE $\mathbf{X}_\text{left} = (x_i: x_i^j \leq \textsc{split-value}(node))$
        \STATE $\mathbf{X}_\text{right} = (x_i: x_i^j > \textsc{split-value}(node))$
        \FOR{$U \in \mathcal{U}$}
            \IF{$j \in U$}
                \STATE $\mathbf{\hat{M}}(\mathbf{X}, U) \gets C_{\text{left}} \mathbf{\hat{M}}_\text{left}(\mathbf{X},U) + C_{\text{right}} \mathbf{\hat{M}}_\text{right}(\mathbf{X},U)$
            \ELSE
                \STATE $\mathbf{\hat{M}}(\mathbf{X}_\text{left}, U) \gets \mathbf{\hat{M}}_\text{left}(\mathbf{X}_\text{left},U)$
                \STATE $\mathbf{\hat{M}}(\mathbf{X}_\text{right}, U) \gets \mathbf{\hat{M}}_\text{right}(\mathbf{X}_\text{right},U)$
            \ENDIF
        \ENDFOR
    \ENDIF
    \STATE {\bfseries Return:} $\mathbf{\hat{M}}(\mathbf{X},\mathcal{U})$
\end{algorithmic}
\end{algorithm}

\subsection{Random Planted Forest Algorithm}\label{sec:rpf_naive}
For the random planted forest (rpf) algorithm, we use a different approach. By slightly altering the representation of an rpf in \cite{hiabu2020random}, the result of an rpf is given by a set $$\hat m^{(0)}=\{\hat m^{(0)}_{S,b} | S\subseteq \{1,\dots,d\},\ b\in\{1,\dots,B\}\},$$ where each estimator $\hat m^{(0)}_{S,b}$ can be represented by a finite partition defined by an $|S|$-dimensional grid (leaves) and corresponding values. Thus, we start with
\begin{itemize}
    \item a grid $G_k=\{x_{k,1},\dots,x_{k,t_k}\}$ for each coordinate $k\in\{1,\dots,d\}$,
    \item for each $S\subseteq \{1,\dots,d\}, b\in\{1,\dots,B\}$ an array representing the value of $m^{(0)}_{S,b}(x)$ for each coordinate $x\in\times_{k\in S}G_k$. Here $x$ is considered to be the bottom left corner of a hyperrectangle.
\end{itemize}
Note that every tree-based algorithm can be described in such a manner. Given an estimator $\hat{p}_S$ and using this representation, directly calculating \eqref{sol} is simple, where for each combination of sets $U,T\subseteq \{1,\dots,d\}$ with $U\subseteq T$, we only need to calculate the term $\int \hat m^{(0)}_{T,b}(x_T) \hat p_U(x_U)\mathrm dx_U$ once and then add/subtract it to the correct estimators $\hat{m}^*_S$. See Algorithm~\ref{alg:naive_rpf} for the full algorithm in pseudo code.

\begin{algorithm}[H]
  \caption{\sc Naïve rpf algorithm}
  \label{alg:naive_rpf}
\begin{algorithmic}
    \STATE {\bfseries Procedure:} \textsc{decompose}($\mathbf{X}$, $\hat{m}(x)$)
    
    \STATE {\bfseries Input:} Dataset to be explained $\mathbf{X} \in \mathbb{R}^{n \times d}$, tree-based model with initial decomposition $\hat{m}_{S,b}^{(0)}(x)$ for $S\subseteq\{1,\dots,d\}$ and trees $b=\{1,\dots,B\}$, estimator $\hat{p}_S$
    \STATE {\bfseries Output:} Components $\hat{m}_S(x_i)$ for all $S \subseteq \{1,\dots, d\}$ and $x_i \in \mathbf{X}$
    \FOR{$i \in 1,\dots,n$}
        \FOR{$S \subseteq \{1,\dots, d\}$}
            \STATE $\hat{m}_S(x_i) \gets 0$
        \ENDFOR
        \FOR{$b \in \{1,\dots, B\}$}
            \FOR{$T \subseteq \{1,\dots, d\}$}
                \FOR{$U \subseteq T$}
                    \STATE $\mathrm{update}_{T,U} \gets \int \hat m^{(0)}_{T,b}(x_{i,T\backslash U},x_U) \hat p_U(x_U)\mathrm dx_U $
                    \FOR{$S: T\setminus S \subseteq U,\ S\subseteq  T$}
                        \STATE $\hat{m}_S(x_i) \gets \hat{m}_S(x_i) + (-1)^{\abs{S}-\abs{T\setminus U}} \mathrm{update}_{T,U}$
                    \ENDFOR
                \ENDFOR
            \ENDFOR
        \ENDFOR
        \STATE $\hat{m}_S(x_i) \gets \hat{m}_S(x_i)/B$
    \ENDFOR
    \STATE {\bfseries Return:} $\hat{m}_S$
\end{algorithmic}
\end{algorithm}

For the calculation of an estimator $\hat{p}_S$ in our simulations we used the following. For each $S\subseteq \{1,\dots,d\}$, let $a_S(x)$ be the number of data points residing in the hyperrectangle with bottom left corner $x$ for each coordinate $x\in\times_{k\in S}G_k$. For $|S|$-dimensional $y$ we then set 
\[
    \hat{p}_S(y)=\frac{a_S(x_y)}{\sum_{x\in \times_{k\in S}G_k} a_S(x)}\frac{1}{\mathrm{vol}(x)},
\]
where $x_y$ is the coordinate of the bottom left corner of the hyperrectangle which includes $y$ and $\mathrm{vol}(x)$ is the volume of the hyperrectangle corresponding to $x$. Using this estimator, the updating function in the algorithm simplifies to
\[
    \mathrm{update}_{T,U} = \sum_{x_U\in\times_{k\in U}G_k} \hat m^{(0)}_{T,b}(x_{i,T\backslash U},x_U) \hat p_U(x_U). 
\]

\section{COMPUTATIONAL COMPLEXITY}

To address computational complexity to calculate all components of the decomposition, we performed a runtime comparison of our proposed algorithm with alternative methods to calculate TreeSHAP interactions. 
See Figure~\ref{fig:runtime} for the results.
We find that, to calculate 2-way interactions, our method is very competitive with the alternatives, i.e. our runtime is similar to that of the TreeSHAP implementation included in the \texttt{xgboost} library, which is also used by the official Python implementation of SHAP (\url{https://github.com/slundberg/shap}).

\begin{figure}
    \centering
    \includegraphics[width=\linewidth]{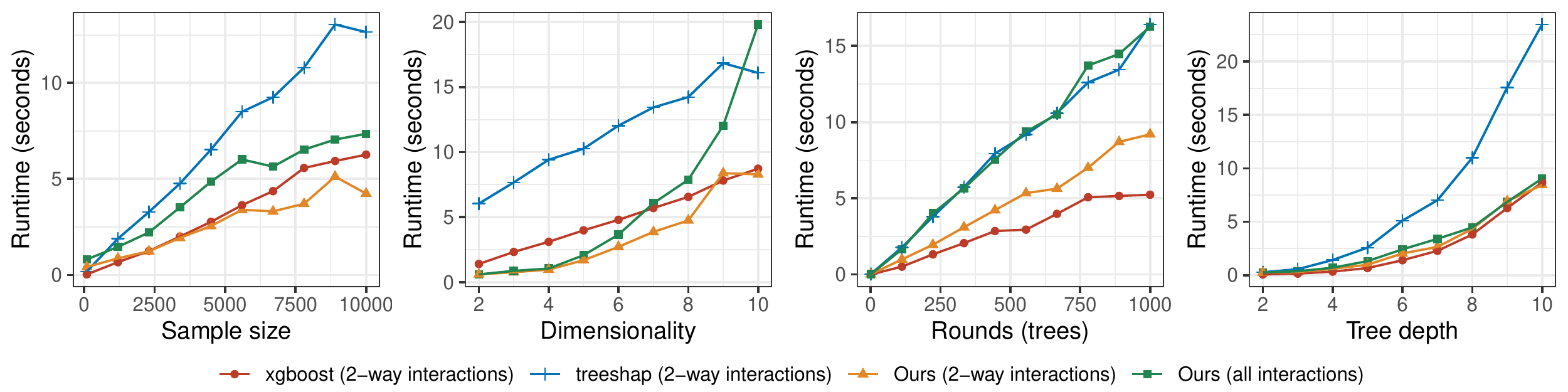} 
    \caption{Runtime comparison. Runtime in seconds to calculate 2-interaction SHAP with different implementations, by sample size, dimensionality, number of rounds/trees and tree depth. In addition, runtime to calculate all interactions with our proposed method (green).}
    \label{fig:runtime}
\end{figure}

\section{EXPERIMENTS WITH RANDOM PLANTED FOREST}

This section shows the simulation results when using the \textit{random planted forest algorithm} as an estimation procedure. First of all, the results considering the motivating example from Section \ref{sec:example} are given in Figure \ref{fig:rpf_simple_example}. The following subsections include the results of experiments which where discussed in Section \ref{Experiments}. 

\begin{figure}
    \centering
    \includegraphics[width=\linewidth]{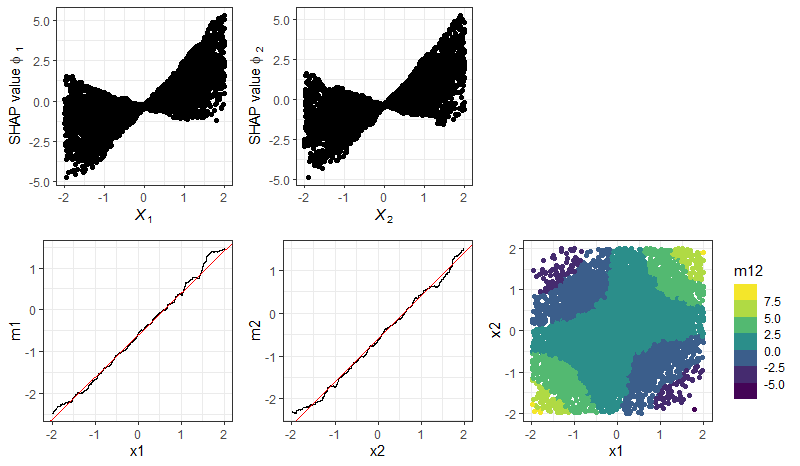} 
    \caption{Simple example. SHAP values (top row) and functional decomposition (bottom row) of a \textit{random planted forest} model of the function $m(x_1,x_2)=x_1+x_2 + 2x_1x_2$. The red lines in the bottom row represent the SHAP values of the true function.} 
    \label{fig:rpf_simple_example}
\end{figure}
\subsection{Global Explanations}\label{sec:bike_rpf}

Figure \ref{fig:bike_rpf} includes the results discussed in Section \ref{sec:bike} when considering the \textit{random planted forest algorithm} instead of \textit{xgboost}.

\begin{figure}[htbp]
    \centering
    \includegraphics[width=1\linewidth]{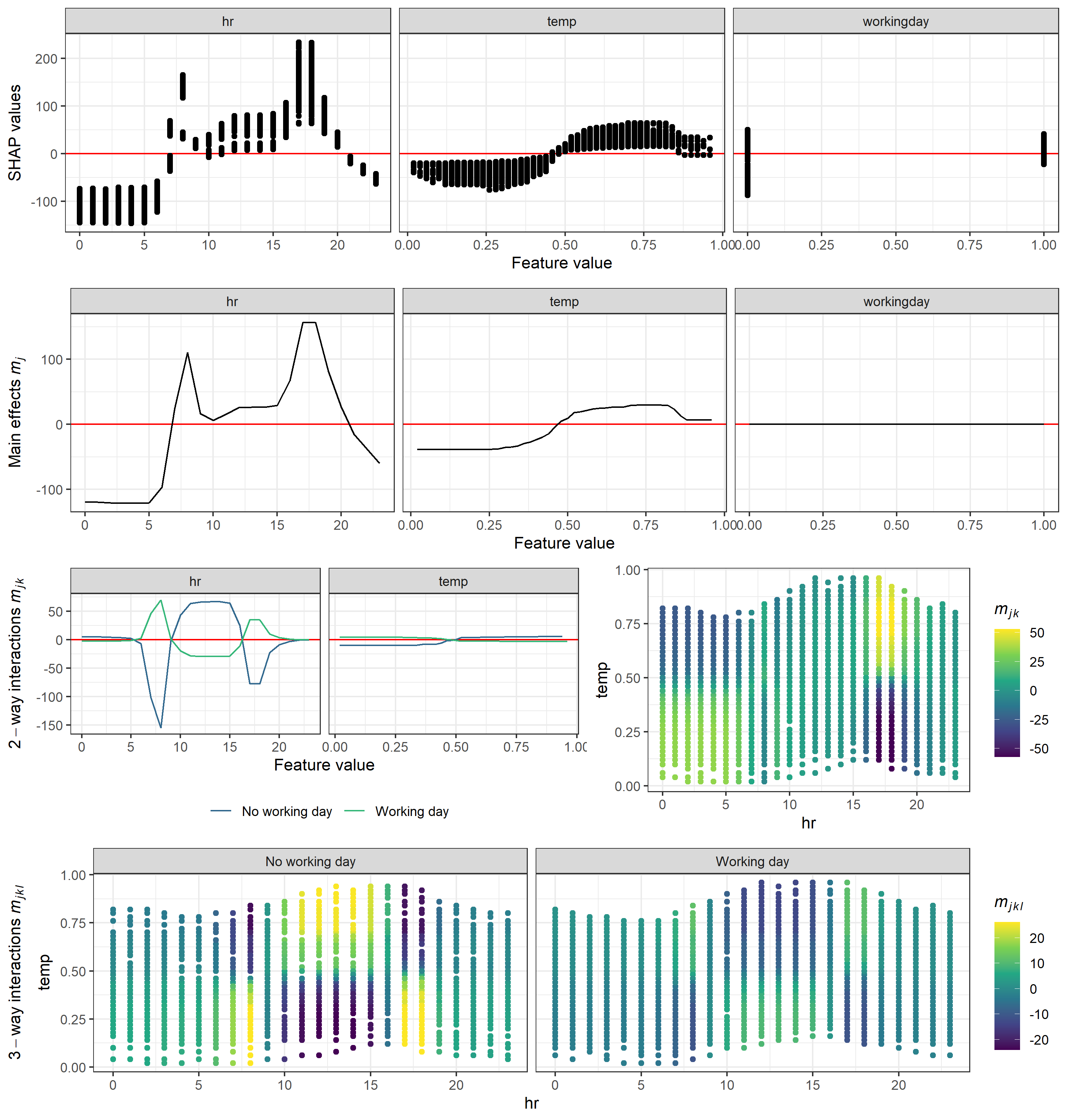}
    \caption{Bike sharing example (\textit{random planted forest}). SHAP values (top row),  main effects (second row), 2-way interactions (third row) and 3-way interactions (bottom row) of the features \textit{hour of the day} (hr, 0-24 full hours), \textit{Temperature} (temp, normalized to 0-1) and \textit{working day} (workingday, 0=no, 1=yes) of the bike sharing data.}
    \label{fig:bike_rpf}
\end{figure}

\subsection{Feature Importance}\label{sec:expvim_rpf}

Figure \ref{fig:vim_rpf} includes the results discussed in Section \ref{sec:expvim} when considering the \textit{random planted forest algorithm} instead of \textit{xgboost}.

\begin{figure}[htbp]
    \centering
    \includegraphics[width=\linewidth]{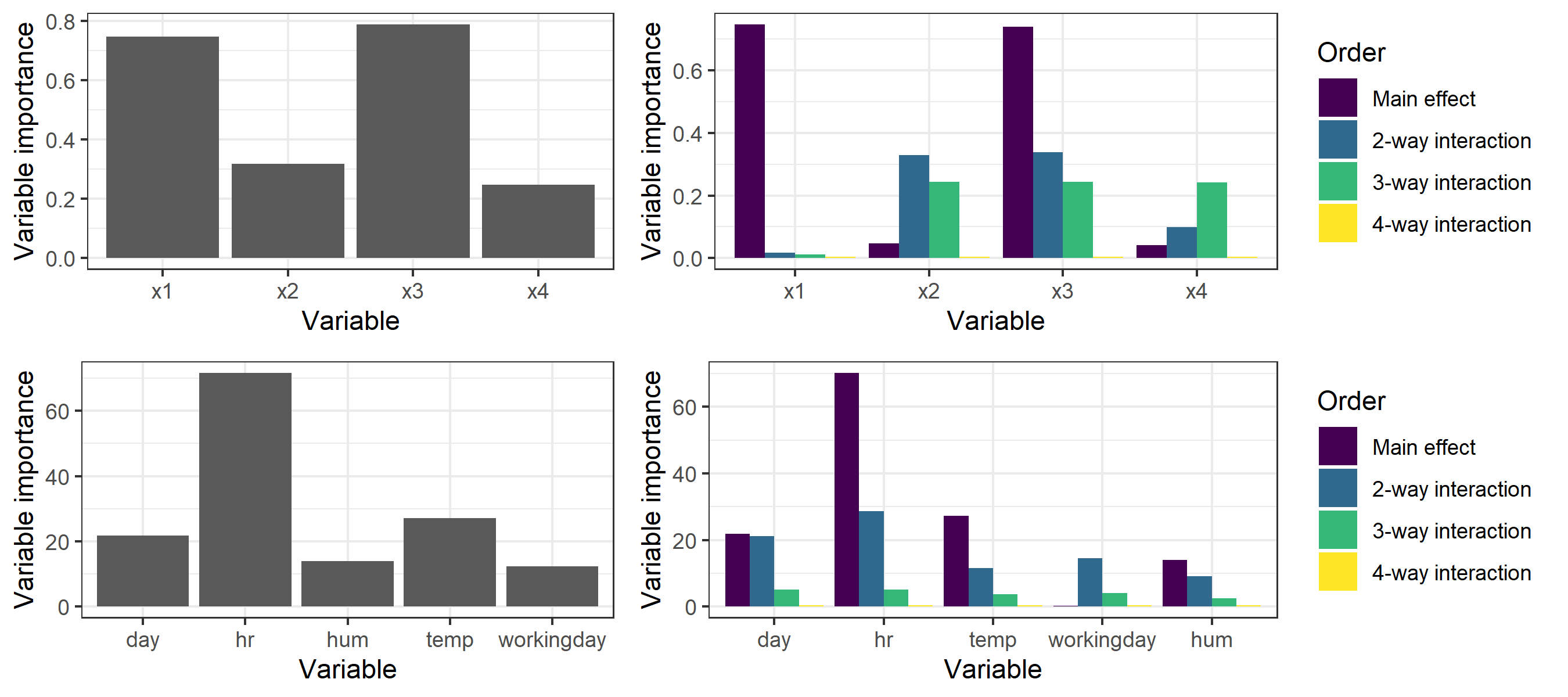}
    \caption{Feature importance (\textit{random planted forest}) for the function $m(x) = x_1 + x_3 + x_2 x_3 - 2 x_2 x_3 x_4$ (top row) and the bike sharing data from Section~\ref{sec:bike} (bottom row)  based on SHAP values (left column) and our functional decomposition separately for main effects and interactions of different orders (right column). } 
    \label{fig:vim_rpf}
\end{figure}

\subsection{Post-hoc Feature Removal}\label{sec:debias_rpf}

Figure \ref{fig:dediscr_rpf} includes the results discussed in Section \ref{sec:debias} when considering the \textit{random planted forest algorithm} instead of \textit{xgboost}.

\begin{figure}[htbp]
    \centering
    \includegraphics[width=\linewidth]{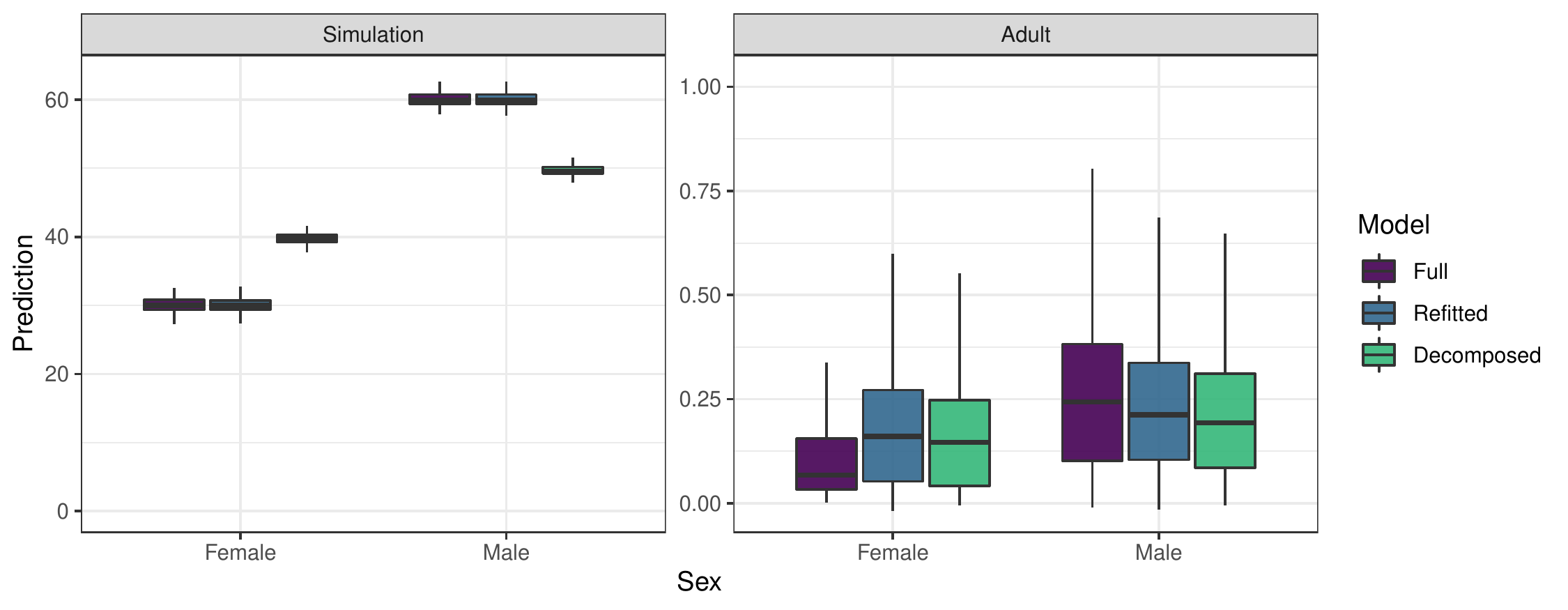}
    \begin{tabular}{lrrr}
    \toprule
    Setting & \multicolumn{3}{c}{Median difference} \\ 
    & Full & Refitted & Decomposed \\
    \midrule
    Simulation & 29.90 & 29.91 & 9.84 \\
    Adult & 0.18 & 0.052& 0.047 \\ 
    \bottomrule
    \end{tabular}   
    \caption{Post-hoc feature removal (\textit{random planted forest}). Predictions in a simulation (left) and the \textit{adult} dataset for males and females of the full model, a refitted model without the protected feature \textit{sex} and a decomposed model where the feature \textit{sex} was removed post-hoc. The table below shows the median differences between females and males for the three models.}
    \label{fig:dediscr_rpf}
\end{figure}

\section{COMPUTING ENVIRONMENT}
A 64-bit Linux platform running Ubuntu 20.04 with an AMD Ryzen Threadripper 3960X (24 cores, 48 threads) CPU and 256 GByte RAM was used for all computations with R version 4.1.2.

\end{document}